\documentclass[conference]{IEEETran}

\usepackage[utf8]{inputenc} 
\usepackage[T1]{fontenc}    
\usepackage{url}            
\usepackage{booktabs}       
\usepackage{amsfonts}       
\usepackage{nicefrac}       
\usepackage{microtype}      
\usepackage{amsmath}		
\usepackage{amsthm}			
\usepackage{amssymb}		
\usepackage{multicol}
\usepackage{graphicx}	
\usepackage{subfig}
\usepackage{tikz}
\usepackage{mathtools}
\usepackage{bbm}
\usepackage{braket}
\usepackage{subfig}
\usepackage{pgfplots}
\usepgfplotslibrary{fillbetween}
\pgfplotsset{width=7cm,compat=newest}

\usetikzlibrary{bayesnet}
\usetikzlibrary{shapes,arrows,positioning,calc}
\usetikzlibrary{datavisualization}
\usetikzlibrary{datavisualization.formats.functions}
\usetikzlibrary{arrows}
\usetikzlibrary{positioning,fit}

\newcommand{\dtv}{\bar{\delta}_{TV}(p, \hat{p})}

\newtheorem{thm}{Theorem}
\newtheorem{lem}{Lemma}
\newtheorem{cor}{Corollary}

\newtheorem{defn}{Definition}

\newcommand{\mc}[1]{\mathcal{#1}}
\newcommand{\mb}[1]{\mathbb{#1}}
\newcommand{\bmat}[1]{\begin{bmatrix} #1 \end{bmatrix}}

\DeclarePairedDelimiterX{\inp}[2]{\langle}{\rangle}{#1, #2}

\begin{document}
\title{Analyzing Data Selection Techniques with Tools from the Theory of Information Losses}

\author{\IEEEauthorblockN{Brandon Foggo}
\IEEEauthorblockA{\textit{Department of Electrical and Computer Engineering} \\
\textit{University of California, Riverside}\\
Riverside, USA \\
bfogg001@ucr.edu}
\and
\IEEEauthorblockN{Nanpeng Yu}
\IEEEauthorblockA{\textit{Department of Electrical and Computer Engineering} \\
\textit{University of California, Riverside}\\
Riverside, USA \\
nyu@ece.ucr.edu}
}

\IEEEoverridecommandlockouts
\IEEEpubid{\makebox[\columnwidth]{978-1-6654-3902-2/21/\$31.00~\copyright2021 IEEE \hfill} \hspace{\columnsep}\makebox[\columnwidth]{ }}
\maketitle
\tikzset{
		block/.style = {draw, fill=white, rectangle, minimum height=1em, minimum width=3em},
		tmp/.style  = {coordinate}, 
		sum/.style= {draw, fill=white, circle, node distance=1cm},
		input/.style = {coordinate},
		output/.style= {coordinate},
		pinstyle/.style = {pin edge={to-,thin,black}
		}
	}



\begin{abstract}
In this paper, we present and illustrate some new tools for rigorously analyzing training data selection methods. These tools focus on the information theoretic losses that occur when sampling data. We use this framework to prove that two methods, Facility Location Selection and Transductive Experimental Design, reduce these losses. These are meant to act as generalizable theoretical examples of applying the field of Information Theoretic Deep Learning Theory to the fields of data selection and active learning. Both analyses yield insight into their respective methods and increase their interpretability. In the case of Transductive Experimental Design, the provided analysis greatly increases the method's scope as well. 
\end{abstract}

\begin{IEEEkeywords}
Information Losses, Training Data Selection, Learning Theory, Facility Location, Transductive Experimental Design.
\end{IEEEkeywords}

\section{Introduction}
Machine learning based classifiers are influenced heavily by the quality of the labeled data they are trained on. But finding high quality samples to label can be challenging, and labelling too many samples can be expensive. This paper provides a new framework for analyzing data labelling methods based on information theory. The framework studies a quantity called information losses, which is a measure of quality on representations learned from the labeled samples.

Information losses have been the subject of some rigorous theoretical work \cite{shamir2010learning, Foggo2019InformationLI}. In reference \cite{Foggo2019InformationLI} in particular, there exists a useful theorem which decomposes information losses as a product of two terms - one depending primarily on architecture, and the other depending primarily on the dataset used to train the classifier. The work provided an analysis of the latter term in the case of \emph{randomly} selected data. 

In this paper, we extend this theory to the case of \emph{non-randomly} selected data. First, we will study the relationship of information losses to two popular existing training data selection techniques - facility location function selection (FLFS) and transductive experimental design (TED). In each analysis, we will bound information losses in terms of the quantities minimized under these methods. 

These analyses are useful for three reasons. First, they provide generalizable examples for the application of Information Theoretic Machine Learning Theory (ITMLT) to the field of training data selection and active learning. Since ITMLT has shown itself to be a useful subfield of machine learning theory, generalizing these analyses can lead to the development of new, powerful, and interpretable training data selection techniques. Secondly, these analyses will bring further knowledge into the methods being studied, giving us insight into which situations these methods are particularly suitable. 

Finally, the analysis of transductive experimental design in terms of information losses significantly widens the method's scope. Previously, the application of TED has been limited to linear and kernel regression problems with Gaussian noise - the optimized terms are directly derived from the error covariances that arise from solutions to these problems. Of course, one could apply these techniques to other problems (e.g. classification) as well, but the application would not be theoretically justified. This paper will show that these same terms naturally reduce information losses, which are themselves linked to the performance of representations well beyond those appearing in linear and kernel regression problems.

\section{Related Work}
The subject of training data selection is extensive. We will consider a coarse division of the field. On one side of this divide is batch mode learning, which selects data all at once. Methods on the batch mode side include the collection of literature on sensor placement \cite{krause2008near, guo2010active}, facility location based methods \cite{sener2018active}, and transductive experimental design \cite{yu2006active}. On the other side of the divide is active learning, which selects new data in iteration by training a new classifier on the currently selected data. Most methods of this type follow from a powerful idea: label the data points that our current classifier is most uncertain of \cite{melville2004diverse, seung1992query, joshi2009multi, elhamifar2013convex, gal2017deep, ganti2012upal, hoi2006large, kremer2014active}. Much can be found in comprehensive texts \cite{settles2012active, tong2001active} and literature surveys \cite{settles2012activesurvey}. While most work in the field of training data selection falls on the active learning side, our framework is applicable to both parts of the division. 

The field is ripe with active learning algorithms that are highly justified within the classical/PAC statistical learning theory \cite{bousquet2003introduction}. Beginning with the CAS algorithm \cite{cohn1994improving}, and being subsequently improved upon in terms of applicability \cite{balcan2009agnostic, dasgupta2008general, beygelzimer2008importance, beygelzimer2010agnostic}, this branch of work rigorously derives algorithms which obtain label complexities, for seperable data, of ${O(\theta d~log~\frac{1}{\epsilon})}$ where $d$ is the $VC$ dimension of the hypothesis space, $\epsilon$ is the desired classification error, and $\theta$ is a useful quantity called the disagreement coefficient of the dataset/hypothesis space pair \cite{hanneke2007bound}. This is an exponential improvement over the label complexity required of random labelling, which needs $O(\frac{d}{\epsilon})$ labels for the same error rate under the same classical learning theory.  

Some early work in the above path even uses information theoretic notions \cite{freund1997selective, gilad2006query}. Specifically, they maintain a probability distribution over the hypothesis space, and data is selected such that the entropy of that distribution is minimized when conditioned on the event ${\{h: h \text{ is consistent with the labelled data}\}}$. Unfortunately, this notion of information is not placed on the class/representation variables themselves, and so they cannot use Fano's inequality in assessing their complexity - instead, they also rely on classical learning theory, obtaining complexities again on the order of $O(d~log~\frac{1}{\epsilon})$ while having the additional complication of needing to maintain and sample from a sequence of posterior distributions on the hypothesis space. 

While the above analyses are fantastic for machine learning algorithms which conform to classical learning theory, we run into problems when we attempt to adapt them to deep learning methods. This is because classical learning theory does not appear to predict the empirical effectiveness of deep learning methods. For example, while the size of a network grows, $d$ increases quite quickly, but the label complexity of the learner drops in experiment, even in the randomly selected case. That is, experimentally observed label complexities in deep learning are far smaller than those predicted by methods in classical learning theory.  

Thus we turn to a promising emerging field of learning theory which links deep learning to information theory \cite{tishby2015deep, shamir2010learning, Foggo2019InformationLI, shwartz2017opening, achille2017emergence}. Many of those active learning methods derived from classical learning theory may be analyzed with this new framework, perhaps giving more satisfying label complexities when applied to deep learning. To our knowledge, there is no previous work in data selection theory which employs this more modern theory of deep learning. 

\section{Notation}
The following notations will be used throughout this text. 

We will only work two simple types of spaces, real vector spaces under the standard norm with the Lebesgue measure on the Borel $\sigma$-algebra, and discrete spaces with the counting measure on the discrete $\sigma$-algebra. We assume that all random variables considered have associated distribution functions and that they all share a base sample space $\Omega$. When we say $A$ is a random variable on $\mc{A}$, we mean that $A: \Omega \to \mc{A}$ is a measurable function. For any considered random variable $A$, we will denote its distribution function as $p_{A}$. For any tuple of considered random variables $(A_1, A_2, \cdots)$, we will denote their joint distribution as $p_{A_1A_2\cdots}$. Conditional distributions will be denoted similarly, but with a bar in between the variables. For example, the distribution of $A_2$ conditioned on variable $A_1$ will be denoted as $p_{A_2|A_1}$. 

Let $\mc{X}$ denote a finite dimensional real vector space with the standard norm and let $\mc{Y}$ be a discrete space. Let $X: \Omega \to \mc{X}$ be a random variable on $\mc{X}$ and $Y: \Omega \to \mc{Y}$ a random variable on $\mc{Y}$. For $i=1,2,\cdots, N$, let $(x_i, y_i) \in \mc{X}\times\mc{Y}$ and let $\mc{D}_{XY} = \left((x_1, y_1), (x_2, y_2), \cdots, (x_n, y_n) \right)$ denote a fixed sample of size $N$ consisting of those points. Let $\mc{D}_X = (x_1, x_2, \cdots, x_N)$ and $\mc{D}_Y = (y_1, y_2, \cdots, y_N)$. Let $\mb{P}_{\mc{D_{XY}}}$, $\mb{P}_{\mc{D_{X}}}$, and $\mb{P}_{\mc{D_{Y}}}$ denote the empirical measures associated with $\mc{D}_{XY}$. These $\mc{D}$ sets are meant to represent a full dataset. We will also consider subsets of this `full' dataset as training samples, which we will denote with the letter $\mc{S}$. However, we will not fix any particular training set since this is a paper on training data selection. 

\section{Information Losses} \label{losses}
\subsection{Measuring Informativeness}
Picking informative data points is difficult. This is partially because it is difficult to define the `informativeness' of a data-point in the first place. In this section, we will construct a measure which quantities the informativeness of a sample directly through that sample's potential to create meaningful representations that are predictive of $Y$. This will arise from a generalization of a discrete-valued approach from literature \cite{shamir2010learning}. We first give an intuitive, non-rigorous outline of this discrete procedure before moving on to our rigorous generalization. 

In such an approach, we first consider the probability mass function $p(x,y)$ and the histogram $\hat{p}(x,y)$ obtained from a sampled dataset. Next, we consider any random variable $Z=g(X)$ (where $g$ may be either deterministic or stochastic). We then  construct the information theoretic quantity ${I(Y,Z^*)}$ - the mutual information between $Z$ and $Y$. If we then consider some `best possible' random variable $Z^*$ (in terms of $I(Y;Z)$) that we could construct, say, with infinite data, and also consider the `best possible' random variable $\hat{Z}$ (again in terms of $I(Y;\hat{Z})$) that we can construct with our sample, then the difference $|I(Y;Z^*) - I(Y;\hat{Z})|$ is then a measure of the sample's quality. 

For our rigorous generalization to a non-discrete feature space, we will need a component to replace the histogram $\hat{p}(x,y)$ since this won't be reliably constructable. We do this by essentially considering \emph{any} replacement distribution $\hat{p}_{Y|X}$ in the above process, and then specifying to one which we obtain in a standard way from some machine learning algorithm. A lot of this setup will come from reference \cite{Foggo2019InformationLI}. However, our setup here is slightly more general, and a bit more clear than that of the reference. 
\begin{defn}\label{defn:iloss1}
    Let $\mc{Z}$ be a finite dimensional real vector space with the standard norm and let $Z$ be a random variable on $\mc{Z}$ satisfying the Markov chain ${Y-X-Z}$ (i.e., $Y$ and $Z$ are conditionally independent given $X$). Let $\hat{Y}$ be a random variable on $\mc{Y}$ satisfying the Markov chain $\hat{Y}-X-Z$. We denote the type one information loss between $Y$ and $\hat{Y}$ associated with $Z$ as the quantity:
    \begin{align}
        I^{(1)}_{Loss}(Y \to \hat{Y}; Z) = |I(Y;Z) - I(\hat{Y}; Z)|
    \end{align}
\end{defn}

\begin{defn}\label{defn:iloss2}
    Let $\mc{Z}$, $Z$, and $\hat{Y}$ be as they were in definition \ref{defn:iloss1}. Let $\epsilon>0$. Denote as $Z_{\epsilon}^*(t)$ the set of random variables on $\mc{Z}$ that are at most ${\epsilon \text{-suboptimal}}$ for the following optimization problem: 
    \begin{align*}
    \underset{Z^*: Y-X-Z^*}{sup~}&{I(Y;Z^*)}  \\
    \text{subject to }& I(X;Z^*) \leq I(X;Z)
    \end{align*}

    Then the type two information loss between $Y$ and $\hat{Y}$ associated with $Z$ is given by the quantity:
    \begin{align}
        I^{(2)}_{Loss, \epsilon}(Y \to \hat{Y}; Z) = \underset{Z^* \in Z^*_{\epsilon}, \hat{Z} \in Z_{\epsilon}(\hat{Y})}{\sup~}|I(Y;Z^*) - I(Y; Z)|
    \end{align}

\end{defn}

On these information losses, we have the following lemma due to reference \cite{Foggo2019InformationLI}:
\begin{lem} \label{lem:info_bound}
Let $\mc{Z}$, $Z$, and $\hat{Y}$ be as they were in definition \ref{defn:iloss1}. Let ${\epsilon>0}$. Let $\hat{p}_{Y|X}=p_{\hat{Y}|X}$. Then $I^{(2)}_{Loss, \epsilon}(Y \to \hat{Y}; Z)$ is less than or equal to the following expression: 
\begin{align}
2 \left( \vphantom{\sum} \dtv I(X;Z) + h_2(\dtv) + \epsilon \right) 
\end{align}
where:
\begin{equation}\label{ctv}
\dtv = \mathbb{E}_{\mathbb{P}_X}\left[ \frac{1}{2}\sum_y \left|p(y|x) - \hat{p}(y|x)\right| \right]
\end{equation}
is called the \emph{conditional total variation} of $\hat{p}_{Y|X}$ from $p(y|x)$
\end{lem}
\begin{proof}
    This is a combination of Lemma 1 and Theorem 1 in reference \cite{Foggo2019InformationLI}.
\end{proof}

Given only the above, we cannot yet consider the term $I^{(2)}_{Loss, \epsilon}(Y \to \hat{Y}; Z)$ to be a measure of sample quality. To obtain such an interpretation, we will still need some additional work. We first note that Lemma \ref{lem:info_bound} holds for \emph{any} pair of random variables $(\hat{Y}, Z)$ satisfying the Markov chains $\hat{Y}-X-Z$, $Y-X-Z$. Of particular interest is the case when $Z, \hat{Y}$ satisfy the requirements in the following definition:
\begin{defn}\label{defn:algorithmic}
Let $\mc{Z}$ and $Z$ be as they were in definition \ref{defn:iloss1}. Let $\mc{S} \subset \mc{D}_{XY}$ be a training dataset. If $Z = \mc{F}(\mc{S})$ where ($\mc{F}$ may be stochastic), then we say that $Z$ is algorithmic on $\mc{S}$. Let $\tilde{Y}$ be another random variable on $\mc{Y}$ defined as a (possibly stochastic) function of $Z$. If we construct $\hat{Y}$ as a function on $X$ with the distribution $p_{\hat{Y}|X} = p_{\tilde{Y}|X}$ (jumping over $Z$), then we say that the pair $(\hat{Y}(\mc{S}), Z(S))$ is algorithmic on $\mc{S}$. 
\end{defn}

When $(\hat{Y}, Z)$ are algorithmic on a training sample $\mc{S}$ (say, through the function $\mc{F}$), then $I^{(2)}_{Loss, \epsilon}(Y \to \hat{Y}, Z)$ can be considered a measure of quality for the pair $(\mc{F}, \mc{S})$. All we need now is to extract the components of this quality measure relating to $\mc{S}$.  

In the case of deep neural networks, we believe that the relationship between $I(X;Z)$ and neural architecture is quite strong. On the other hand, reference \cite{Foggo2019InformationLI} proved a bound on $\dtv$ for randomly selected data in a way that did not depend on neural architecture. Thus we will consider the quality of the dataset to mostly interact with the term $\dtv$, while the quality of the `algorithm' (here, the neural architecture) will mostly interact with the term $I(X;Z)$. Thus we will focus our efforts on the study of $\dtv$. 

\subsection{Links to Classification Accuracy}
We wish to quickly show some analytical links between our quality measure, $I^{(2)}_{Loss, \epsilon}(Y \to \hat{Y}, Z)$, and classification accuracy. Our starting point is related to Fano's inequality \cite{cover2012elements}. To continue, we will need another definition:
\begin{defn}
    We say that $X$ is sufficient for the classification of $Y$ if $I(X;Y)=H(Y)$. 
\end{defn}
We then have:
\begin{lem}
Suppose $X$ is sufficient for the classification of $Y$. Let $(Z, \hat{Y})$ be algorithmic on a training sample $\mc{S}$. Let $\tilde{Y}$ be the random function of $Z$ defined as it was in definition \ref{defn:algorithmic}. Let $E$ be the random variable which takes the value $1$ when $\tilde{Y}=Y$ and $0$ otherwise. Let $p_e = \mathbb{P}(E=1)$. Let $log(|\mc{Y}|-1) \geq H(Y|E=1, Z) \geq \alpha > 0$ (this entails a multi-class problem) denote the level of remaining uncertainty in $Y$ once it is known that the estimator $\tilde{Y}$ is incorrect. Suppose $H(Y|Z)>0$ and let $H(Y|Z)>t>0$. Then $\exists r>0$ such that, if $I(X;Z) > r$:
\begin{equation}
    p_e \leq \frac{I^{(2)}_{Loss, \epsilon}(Y \to \hat{Y}, Z) + t}{\alpha}
\end{equation}

\end{lem}
\begin{proof}   
    The first step in the usual proof of Fano's inequality gives us:
    \begin{align}
        H(Y|Z) = H(E|Z) + p_e H(Y|E=1,Z)
    \end{align}
    
    Short-handing $I(X;Z) = r$, we have that the supremum in definition \ref{defn:iloss2} achieves $I(Y;Z^*) = H(Y) - t'(r)$ where $t'(r)>0$. If $X$ is sufficient, then $t(r) \to 0$ as $r \to \infty$. Then if $r$ is large enough such that $t'(r) < t$, we will have the following equality: ${I^{(2)}_{Loss, \epsilon}(Y \to \hat{Y}, Z) = |H(Y) - I(Y;Z) - t(r)|}$, which is equal to ${H(Y|Z) - t(r)}$. Thus we have:
\begin{align}
    p_e  &= \frac{I^{(2)}_{Loss, \epsilon}(Y \to \hat{Y}, Z) + t(r) - H(E|Z)}{H(Y|E=1, Z)} \notag \\ &\leq \frac{I^{(2)}_{Loss, \epsilon}(Y \to \hat{Y}, Z) + t}{\alpha}
\end{align}
\end{proof}

Given this Lemma, we see that studying $\dtv$ is not only useful for analyzing our training sample's ability to yield good representations, but for analyzing our classifier's probability of error as well.

\section{Facility Location Selection Reduces Information Losses }\label{active}
We will first show that selecting data according to a specific criterion - minimizing the facility location function - hedges risks in information losses. This section is mostly meant to act as a generalizable example of using the information loss framework to show that information losses are easy to deal with and lead to intuitive proofs of validity for a given method in terms of representational quality.

For a general training data selection strategy, we emphasize the goal of finding a naive `test' estimator which is somewhat natural to the strategy. We can then bound the conditional total variation of the `test' estimator relatively easily. This task will often reduce to plain analysis due to the simplicity of the conditional total variation term. Doing this will often give us insight into when a given strategy is useful.

For our example, we take the training data selection strategy which attempts to minimize the following function of the training dataset $S$, ${Z(S) = \mathbb{E}_{\mc{P}_X}\left[\|x - x_i \| \right]}$, where $x_i$ is the nearest neighbor of $x$ in $S$. This method is known as the \textit{facility location function selection method} \cite{drezner2001facility,sener2018active}, and it is a practical, intuitive, all-at-once data selection technique. The goal of this strategy is to pick data points such that, on average, every data point is geometrically close to some training point. 

To analyze this strategy, we will use a `test' estimator which takes into account local information near the training data. This yields the following Thoerem:
\begin{thm}\label{thm:fac}
Let $\mc{X}$ be a bounded subset of ${\mathbb{R}^d}$. Suppose that we have a Lipschitz-continuous, differentiable conditional probability function ${p(y|x): \mathbb{R}^{d} \to \mathbb{R}^{|\mc{Y}|}}$ with Lipschitz coefficient $L$ (maximized over each class variable). Let $\mc{S}$ denote a training dataset indexed by $i$. Let ${\mc{R}_i}$ be the set of points in $\mathbb{R}^d$ whose nearest neighbor in $\mc{S}$ is $x_i$ and consider the following `neighbors' estimator of $p(y|x)$: ${\hat{p}_{nn}(y|x) = p(y|x_i),~ x\in \mc{R}_i}$. Then $\hat{p}_{nn}$, and therefore any algorithm beating it (in terms of conditional total variation), has ${ \lim_{Z(S) \to 0} \frac{\dtv}{Z(S)} \leq \frac{L|\mc{Y}|}{2}}$.
\end{thm}

\begin{proof}
We can linearly approximate $p(y|x)$ in each region $\mc{R}_i$. The absolute error between $p(y|x)$ and $\hat{p}_{nn}(y|x)$ in this region is given, for all $y \in \mc{Y}$, by:
\begin{equation}
    |p(y|x)- \hat{p}_{nn}(y|x)| =  \left|\nabla p(y|x_i)^T(x - x_i) + o(\|x-x_i \|)\right|
\end{equation}
We can then compute the expected conditional total variation between $p$ and $\hat{p}_{nn}$ by summing the contributions from each region and each class variable to obtain:
\begin{equation}
    \frac{1}{2}\sum_y\sum_i \int_{\mc{R}_i} \{|\nabla p(y|x_i)^T(x-x_i) + o(\|x-x_i\|)|\}d\mc{P}_X
\end{equation}
which, by the Cauchy-Schwartz inequality and triangle inequality can be upper bounded with the following expression:
\begin{align}
&&\frac{1}{2}\sum_i\sum_y\ \Big \{| \nabla p(y|x_i)\| &\int_{\mc{R}_i}  \|x-x_i\| d\mc{P}_X \notag \\ &&+ &\int_{\mc{R}_i}|o(\|x-x_i\|)|d\mc{P}_X \Big \}
\end{align}
which is itself upper bounded by:
\begin{equation}
    \frac{L|\mc{Y}|}{2} Z(S) + \frac{1}{2}|\mc{Y}|\sum_i\int_{\mc{R}_i}|o(\|x-x_i\|)|d\mc{P}_X
\end{equation}

Denote ${\eta_S: \mathbb{R}^d \to \mathbb{R}^d}$ as the function which takes $x$ to its nearest neighbor in $S$. We then have:
\begin{align}
    \frac{\delta_{TV}(p, \hat{p}_{nn})}{Z(S)} &\leq  \frac{L|\mc{Y}|}{2} +  \frac{|\mc{Y}|}{2}\frac{\int_{\mathbb{R}^d}|o(\|x-\eta_S(x)\|)|d\mc{P}_X}{\int_{\mathbb{R}^d}\|x-\eta_S(x)\| d\mc{P}_X} \notag \\
    &\leq \frac{L|\mc{Y}|}{2} + \frac{|\mc{Y}|}{2}\int_{\mathbb{R}^d}\frac{|o(\|x-\eta_S(x)\|)|}{\| x-\eta_S(x)\|}d\mc{P}_X 
\end{align}
(For the last inequality, let ${X=\frac{|o(\|x-\eta_S(x)\|)|}{\| x-\eta_S(x)\|}}$ and ${Y = \|x - \eta_S(x)\|}$ in the Cauchy-Schwartz inequality). 

Now, since ${\|x-\eta_S(x)\| >0}$, ${Z(S) \to 0}$ implies ${\|x-\eta_S(x)\| \to 0}$ on all but a set of measure zero (this follows from the bounded convergence theorem). Thus ${\frac{|o(\|x-\eta_S(x)\|)|}{\| x-\eta_S(x)\|} \to 0}$ almost surely, completing the proof.
\end{proof}

Put another way, Theorem \ref{thm:fac} states that $\dtv$ is bounded above by a function which asymptotically behaves as ${\frac{1}{2}L|\mc{Y}|Z(S)}$, which is linear in $Z(S)$. As such, $Z(S)$ acts linearly on $\dtv$, and therefore on our representational quality and probability of error.

If we wish to drop the Lipschitz condition in our hypothesis, then we can use the proof of Theorem \ref{thm:fac} to obtain the following corollary:
\begin{cor}\label{corr:fac}
Take all of the assumptions of Theorem \ref{thm:fac}, but remove the assumption that $p(y|x)$ is Lipschitz-continuous. Let ${\tilde{Z}(S) = \sum_y\sum_{\mc{R}_i} \|\nabla p(y|x_i)\| \cdot \mathbb{E}_{\mc{P}_X}\left[1_{x\in \mc{R}_i}\cdot \|x - x_i \| \right]}$. Then ${\lim_{\tilde{Z}(S) \to 0} \frac{\delta_{TV}(p, \hat{p})}{\tilde{Z}(S)} \leq \frac{1}{2}}$. 
\end{cor} 

This corollary may be particularly useful if we are capable of estimating the gradient $\|\nabla p(y|x_i)\|$ at each training data point, in which optimization of $\tilde{Z}(S)$ is possible and leads to an augmented facility location function selection method. If, however, we do take our Lipschitz assumption, and use the standard facility location function selection, then we may expect this method to be most effective when dealing with functions that vary quite rapidly (making $L$ large and therefore the marginal improvement obtained by decreasing the facility location value).

\section{Transductive Experimental Design Reduces Information Losses}

\subsection{Transductive Experimental Design}
Our second example focuses on a technique known as transductive experimental design \cite{yu2006active}. We would like to note that this example is probably more important than the previous, since it has a more immediate applicability. This is because the analysis significantly expands the scope of the method. While the original intention of Transductive Experimental Design is to reduce error variances in linear and kernel regression problems, our analysis shows that the same technique results in higher quality representations in general - and is thus applicable to a multitude of models. Furthermore, the field of research stemming from the original work on TED \cite{yu2008trnon, sindhwani2009uncertainty, cai2011manifold} can be extended to a more general setting via our analysis as well.

We will begin by quickly reviewing the Transductive Experimental Design method. We will assume that we have a set $\mc{D}$ consisting of $N$ unlabelled data points. We wish to label a subset of $\mc{D}$, denoted $S$, of these datapoints with cardinality ${M < N}$.

The method originates as an improvement over the techniques in the field of Optimal Design. In optimal design, we consider the regression task of estimating the vector $w$ in the equation ${y = w^Tx + \eta}$ via a regularized $L_2$ loss function $\mc{L}(w) = {\sum^M_{i=1} \|w^Tx_i - y_i \|^2 + \mu \|w\|^2_2}$, where $\eta$ are noise values distributed through $\eta \sim \mc{N}(\eta; 0, \sigma^2)$. When the optimum $\hat{w}$ is taken under this loss function, the estimation error $\hat{w} - w$ has a covariance matrix given by $C = \sigma^2\left(X^TX  + \mu I\right)^{-1}$ where $X$ is the design matrix corresponding to our training data, ${X = \begin{bmatrix} x_1 & x_2 & \cdots & x_M \end{bmatrix}^T}$. Different experimental design procedures exist to optimize different statistics of this covariance matrix by selecting training data points. For example, \emph{A-optimal} design attempts to minimize the trace of $C$, \emph{E-optimal} design attempts to minimize the largest eigenvalue of $C$, and \emph{D-optimal} design attempts to minimize its determinant. 

Transductive expermimental design attempts to take optimal design one step further and consider the covariance matrix of the vector containing the values $\hat{w}^Tx_i - w^Tx_i$ where the index $i$ runs over all points $x \in \mc{D}$ with \emph{un-obeserved} $y$ values. That is, transductive experimental design directly considers the covariance in generalization error. It can be shown that this covariance matrix is given by $VCV^T$ where $V$ is the design matrix corresponding to all of the unlabelled data, and $C$ is the covariance matrix obtained in standard experimental design. This is equivalent to the expression ${\frac{1}{\mu}\left(VV^T + VX^T(XX^T+\mu I)^{-1}XV^T \right)}$. TED then attempts to minimize the trace of this matrix by selecting rows of $V$ to place in rows of $X$. 

This can be kernelized as follows. Given a kernel $k(\cdot, \cdot)$, we can use the kernel trick to identify $VV^T$ as a matrix $K_{VV}$ with values $k(x_r, x_s)$ where $x_r$ and $x_s$ are unlabelled datapoints, identify $VX^T$ as a matrix $K_{VX}$ filled with values $k(x_i, x_j)$ where $x_i$ is unlabelled and $x_j$ is labelled, and identify $XX^T$ in the usual way with a matrix $K_{XX}$ filled with values $k(x_k, x_l)$ where both $x_k$ and $x_l$ are labelled. This leads to minimization of the following term:
\begin{equation} \label{obj:ted}
    Trace\left(K_{VV} -  K_{VX}\left(K_{XX}^{-1} + \mu I \right)^{-1} K_{VX}^T\right)
\end{equation}
Minimizing this expression is the goal of kernel Transductive Experimental Design. 

We will next bound the information loss term, $\dtv$, in a way that is naturally and directly dependent on this trace term. In doing so, we show that optimizing this term via transductive experimental design naturally leads to higher quality representations and lower classification errors despite such applications not being a part of the original scope of the method. For notational simplicity, we will denote the trace term (\ref{obj:ted}) as $TED(\mc{S}, \mc{D}, \mu)$, where, again, $S$ denotes the training dataset and $\mc{D}$ denotes the full dataset. We will also denote as $TED^{\frac{1}{2}}(\mc{S}, \mc{D}, \mu)$ the corresponding term when the matrix inside the trace operation is first subject to an element-wise square-root operation.

\subsection{Bounding Information Losses via the TED objective function}

\subsubsection{Notation and Basic Assumptions}
We assume that the feature space $\mc{X}$ is a finite dimensional Euclidean space. We assume that we have some continuous, symmetric, positive definite kernel function $k(\cdot, \cdot)$ with a corresponding Reproducing Kernel Hilbert space (RKHS) $\mc{H}$ \cite{berlinet2011reproducing}. We will denote the inner product on this RKHS as $\inp{\cdot}{\cdot}_{\mc{H}}$. Let $\mu$ be a regular Borel measure on $\mc{X}$. We will frequently refer the integral operators $T: L_{\mu}^2 \to \mc{H}$ and $T': L_{\mu}^2 \to L_{\mu}^2$, both given by: 
\begin{equation}
    f \mapsto \int k(x,y)f(y)d\mu(y)
\end{equation}
(but with differing co-domains).

$T'$ is a self-adjoint operator, and $T$ is adjoint to the embedding operator ${R: \mc{H} \to L_{\mu}^2}$ given by \cite{belkin2018approximation} ${(Rf)(x) = f(x)}$. That is, ${\braket{f, Tg}_{\mc{H}} = \braket{Rf, g}_{L_{\mu}^2}}$. We will also need an operator $R': L^2_{\mu} \to L^1_{\mu}$ which maps $f$ to itself, but under a different norm. By Mercer's theorem, $T'$ admits a countable set of eigenfunctions, $\{\phi_i\}_{i=1}$, which are $L^2_{\mu}$-orthonormal, $\mc{H}$-orthogonal and have positive decreasing eigenvalues $\{\lambda_i\}$ with $\underset{i \to \infty}{\lambda_i} = 0$.

We will also introduce the following matrix building notation: if $q$ is an index with domain $\{1,2,\cdots,Q\}$ and $p$ is an index with domain $\{1,2,\cdots,P\}$, then ${\begin{bmatrix}a_{p_l} \end{bmatrix}^{l} \triangleq \begin{bmatrix}a_1 &a_2 & \cdots &a_P \end{bmatrix}}$, ${\begin{bmatrix}a_{p_l} \end{bmatrix}_{l} \triangleq \begin{bmatrix}a_1 &a_2 & \cdots &a_P \end{bmatrix}^T}$, and ${\begin{bmatrix}[a_{p_lq_{l'}}]_{l} \end{bmatrix}^{l'} = \begin{bmatrix}[a_{p_lq_{l'}}]^{l'} \end{bmatrix}_{l} = \begin{bmatrix}a_{p_{l}q_{l'}} \end{bmatrix}_{l}^{l'}}$ where the final three matrices are all given by the matrix whose $ij^{th}$ element is $a_{p_iq_j}$, ${1 \leq i \leq P,~ 1 \leq j \leq Q}$. 

We denote three important index maps. The first is $\mc{I}: \{1,2, \cdots, M\}$ which indexes training data points via $x_{i_l}$. The second is $\mc{A}: \{1,2, \cdots, N\}$ which indexes all available data points via $x_{a_l}$. The third is $\mc{U}: \{1,2, \cdots, N-M\}$ which indexes all unlabelled points through $x_{u_l}$.  

Finally, the total variation is equivalent to $1$-norm on a subset of $L^1_{\mu}$ since $\mathbb{E}_{\mu}\left [ |p_{y|x} - \hat{p}_{y|x}| \right]$ \footnote{By writing the total variation in this form, we are implicitly assuming that our problem is 2-class. This is done for notational convenience. We will extend the theory to multiple classes immediately after finishing the 2-class case.}. We note that $p_{y|x}$ is an element of $L_{\mu}^1$, as ${\int \left|p_{y|x}(x) \right| d\mu(x) = p(Y=1) \leq 1 < \infty}$. Since $L^1_{\mu} \subseteq L^2_{\mu}$, we will primarily consider $p_{y|x}$ as an element of $L^2_{\mu}$ and use $R'p_{y|x}$ to view it as an element of $L^1_{\mu}$.

\subsubsection{Approximation Theory and Technical Lemmas}
To estimate $\dtv$ under a selected training set, we will begin with a definition that is useful for bounding the deviations of any function in $\mc{H}$ from its projection onto a given subspace.  
\begin{defn}
Let $V$ be a subspace of $\mc{H}$. Then the \emph{power function} on $V$, denoted $P_V$, is the function whose point-wise evaluation is given by:
\begin{equation}
    P_V(x) = \underset{\|f\|\leq 1}{\text{sup}}\left|f(x) - (\text{proj}_{V}f)(x)\right|
\end{equation}
where $\text{proj}_{V}$ is the orthogonal projection operator onto $V$. 
\end{defn}
We will deal with the particular finite dimensional subspace $V_{S} = \text{Span}\left(\{k(\cdot, x_{i_l}) \}^M_{l=1} \right)$. When $V$ is such a subspace, the projection operator $\text{proj}_{V}$ takes on the following well known result in approximation theory \cite{belkin2018approximation} whose proof we will not repeat here. 
\begin{lem}\label{lemma:fin_dim_proj}
    Let ${ K_{SS} \triangleq \begin{bmatrix}k(x_{i_l}, x_{i_{l'}}) \end{bmatrix}_{l}^{l'}}$, and for any $x \in \mc{X}$, let ${K_{xS}  = K_{Sx}^T \triangleq \begin{bmatrix}k(x, x_{i_{l}}) \end{bmatrix}^{l}}$. Then for all $f \in \mc{H}$, ${(\text{proj}_{V_S}f)(x) = K_{xS}K_{SS}^{-1}\begin{bmatrix} f(x_{i_l})\end{bmatrix}_{l}}$. 
\end{lem}
The next lemma from approximation theory gives us the desired bound on deviations of any function in $\mc{H}$ from its projection onto $V_S$.
\begin{lem}\label{lemma:pf_bound}
    Let $f \in \mc{H}$. Then the following inequality holds: ${|f(x) - (\text{proj}_{V_S}f)(x)| \leq  |P_{V_S}(x)|\|f\|_{\mc{H}}}$.
\end{lem}
If $p_{y|x}$ is in $\mc{H}$ then we can apply this immediately to our problem. But even if it is in $\mc{H}$, it may be high frequency, and so $\|f\|_{\mc{H}}$ may be large. To account for this, we will provide two bounds, one to cover the case when $p_{y|x} \in \mc{H}$ and $\|f\|_{\mc{H}}$ is small, and one to cover all other cases. To cover the other cases, we will decompose $p_{y|x}$ into a part in $\mc{H}$ and a part not in $\mc{H}$ by using the operator $T$ to write:
\begin{align}\label{eqn:error_decomp}
    &R'p_{y|x} - R'R\text{proj}_{V_S} Tp_{y|x}  \notag \\ 
    &\qquad= (R'-R'RT)p_{y|x} + R'RTp_{y|x} - R'R\text{proj}_{V_S} T p_{y|x}
\end{align}
Similarly, we have that:
\begin{align}
   &\|R'p_{y|x} - R'R\text{proj}_{V_S} Tp_{y|x}\|_{L^1_{\mu}} \leq \|(R'-R'RT)p_{y|x}\|_{L^1_{\mu}} & \notag  \\ & \qquad \qquad \qquad + \|R'RTp_{y|x} - R'R\text{proj}_{V_S}Tp_{y|x} \|_{L^1_{\mu}}
\end{align}
We will leave the study of the first term until the end of this subsection, and just denote it as $\epsilon_{\mc{H}}$ for now. Thus we will move to studying the second term, which will be handled, primarily, by lemma  \ref{lemma:pf_bound}. To make any progress, we will need to bound the RKHS norm of $Tp_{y|x}$, which we do in the following lemma: 
\begin{lem}\label{lemma:p_norm}
     ${\|Tp_{y|x}\|_{\mc{H}} \leq p(y=1){\sqrt{Tr(k)}}}$ where $Tr(k)$ is the trace of the operator $T$, i.e. $Tr(k) = \int k(x,x) d\mu(x) = \sum_i \lambda_i$.
\end{lem}
\begin{proof}
\begin{align*}
    \|Tp_{y|x}\|_{\mc{H}}^2 &= \braket{Tp_{y|x} , Tp_{y|x}}_{\mc{H}} = \braket{RT\{p_{y|x}\} , p_{y|x}}_{L^2_{\mu}}  \notag \\
    &= \int \left(\int k(x,x')p_{y|x}(x') d\mu(x')\right) p_{y|x}(x)  d\mu(x) \notag \\
    &= \int \int k(x,x') p_{y|x}(x)p_{y|x}(x') d\mu^2 \notag \\  
    & \left(\text{where } d\mu^2 \triangleq d\mu(x) \otimes d\mu(x')\right) \notag \\
\end{align*}
\begin{align}
    &\leq \sqrt{\int\int k(x, x')^2 d\mu^2 \times \int\int  p^4_{y|x}(x, x') d\mu^2} \notag \\
    & \qquad \qquad \left(\text{where } p^4_{y|x}(x, x')  \triangleq p^2_{y|x}(x)p^2_{y|x}(x')\right) \notag \\
    &\leq \int k(x,x') d\mu(x) \times \sqrt{\int\int   p^4_{y|x}(x, x') d\mu^2} \notag \\
    &\leq \int k(x,x') d\mu(x) \times \int \int p_{y|x}(x)p_{y|x}(x')   d\mu^2 \notag \\
    &=Tr(k)\int p^2_{y|x}(x)d\mu(x) \notag \\
   &\leq Tr(k)\left(\int p_{y|x}(x) d\mu(x))\right)^2 = Tr(k)p^2(y)
\end{align}
\end{proof}
Then combining lemmas \ref{lemma:pf_bound} and \ref{lemma:p_norm}, we conclude that:
\begin{equation}\label{eqn:1norm}
    \left|Tp_{y|x}(x) - \text{proj}_{V_S}Tp_{y|x}(x)\right| \leq \left|P_{V_S}(x)\right|\sqrt{Tr(k)}p(y=1)
\end{equation}

Note that the expectation of the left hand side of (\ref{eqn:1norm}) is equal to the $L^1_{\mu}$ norm ${\|R'RTp_{y|x} - R'R\text{proj}_{V_S}Tp_{y|x} \|_{L^1_{\mu}}}$, the term under study. Since we are interested in \emph{empirical} estimates of this $L^1_{\mu}$ norm, we will now turn to manipulating the empirical expectation of  $|P_{V_S}(x)|$ [the only term that depends on $x$ in (\ref{eqn:1norm})] in a nice form. This nice form follows from one final cited lemma from approximation theory \cite{schaback1993comparison}:
\begin{lem}\label{lem:p_from_k}
    \begin{equation}\label{eqn:p_from_k}
        |P_{V_S}(x)| = \sqrt{K(x,x) - \bmat{k(x, x_{i_l})}^l K_{SS}^{-1}\bmat{k(x, x_{i_l})}_l}
    \end{equation}
\end{lem}

Calculating the empirical expectation of this then immediately lends itself to the TED objective function. The result is as follows:
\begin{lem}\label{lemma:main}
    Let $\hat{\mu}_{\mc{D}}$ be the empirical measure over $\mc{D}$. Let ${ K = \begin{bmatrix} k(x_{j_l}, x_{j_{l'}}) \end{bmatrix}_{l}^{l'}}$. Then:
    \begin{equation}\label{eqn:p_schur}
        \mathbb{E}_{\hat{\mu}_{\mc{D}}}\left[|P_{V_s}(x)| \right] = \frac{1}{N}\text{Trace}\left(\sqrt{K/K_{SS}}\right)
    \end{equation}
    where the notation $X/A$ refers to the Schur complement of $X$ with respect to $A$, and $\sqrt{\cdot}$ refers to taking the element-wise square root of the matrix in its argument.
\end{lem}
\begin{proof}
    From lemma \ref{lem:p_from_k}, we have $|P_{V_S}(x)| = \sqrt{K(x,x) - \bmat{k(x, x_{i_l})}^l K_{SS}^{-1}\bmat{k(x, x_{i_l})}_l}$. Note that, if $x$ is a training data point, $\bmat{k(x, x_{i_l})}_l$ is the $l^{th}$ column of $K_{SS}$, so $K_{SS}^{-1}\bmat{k(x, x_{i_l})}_l = e_{l}$ where $e_{l}$ is the standard unit vector with $1$ at position $l$ and zeroes elsewhere. Then $\bmat{k(x, x_{i_l})}^l K_{SS}^{-1}\bmat{k(x, x_{i_l})}_l$ is equal to $k(x, x)$. Thus, in summing (\ref{eqn:p_from_k}) over the dataset $\mc{D}$, we only need to include terms corresponding to unlabelled data.
    
    Recall that the Schur complement is given by:
    \begin{equation}
        K/K_{SS} = \bmat{k(x_{u_l}, x_{u_{l'}})}^l_{l'} - \bmat{k(x_{u_{l'}}, x_{i_l})}^l_{l'} K_{SS}^{-1}\bmat{k(x_{u_{l'}}, x_{i_l})}_l^{l'}
    \end{equation}
    from which we can see that the $p^{th}$ diagonal element of $\sqrt{K / K_{SS}}$ is equal to $|P_{V_S}(x_{u_p})|$. Summing over the unlabelled points is then equivalent to taking the trace of this matrix. Dividing by the size of $\mc{D}$ ($N$) completes the result. 
\end{proof}

Finally, we will combine all of these lemmas to obtain the following bound:
\begin{equation}\label{eqn:all_together}
    \delta_{\hat{p}_{y=c}}^{emp} \leq \frac{p(y=c)}{N}\sqrt{Tr(k)}Trace \left(\sqrt{K/K_{SS}}\right) + \epsilon_{\mc{H}}
\end{equation}

Before we wrap up this section, we will need to quickly return to the study of $\epsilon_{\mc{H}}$. We have the following lemma:
\begin{lem}\label{lem:eh}
    Let $\{\zeta_j\}$ be a (countable) orthonormal basis for $Null(T)$ (this exists by the separability of $L^2_{\mu}$). Then:
    \begin{equation}
        \epsilon_{\mc{H}}  \leq \sqrt{\sum_i \inp{p_{y|x}}{\phi_i}^2(1-\lambda_i)^2 + \sum_j \inp{p_{y|x}}{\zeta_j}^2}
    \end{equation}
\end{lem}
\begin{proof}
    Since $\{\phi_i\} \cup \{\zeta_j\}$ form a basis of $L^2_{\mu}$, and ${p_{y|x} \in L^2_{\mu}}$, we can write:
    \begin{equation}
        p_{y|x} = \sum_i \inp{p_{y|x}}{\phi_i}\phi_i + \sum_j \inp{p_{y|x}}{\zeta_j} \zeta_j
    \end{equation}
    Then:
    \begin{align}
        p_{y|x} - RTp_{y|x} &= \sum_i \inp{p_{y|x}}{\phi_i}\phi_i + \sum_j \inp{p_{y|x}}{\zeta_j} \zeta_j \notag \\
        &\qquad - \sum_l \inp{p_{y|x}}{\phi_l} \int k(.,y)  \phi_l(\tau) d\mu(\tau) \notag \\
        &= \sum_i \inp{p_{y|x}}{\phi_i}\phi_i + \sum_j \inp{p_{y|x}}{\zeta_j} \zeta_j \notag \\
        & \qquad \qquad - \sum_{i} \lambda_i \phi_i \inp{p_{y|x}}{\phi_i} \notag \\
        &= \sum_i (1-\lambda_i)\inp{p_{y|x}}{\phi_i}\phi_i + \sum_j \inp{p_{y|x}}{\zeta_j} \zeta_j
    \end{align}
    The lemma then follows by noting that $L^1_{\mu}$ norms are bounded by $L^2_{\mu}$ norms and then applying Pythagorean's theorem. 
    
\end{proof}

\subsubsection{Converting to multiple classes}
Converting this bound to multiple classes is as simple as removing the ${p(y=c)}$ term and dividing by $2$. This is because the total variation over multiple classes is given by half the sum of each $L_1$-norm. Thus the ${p(y=c)}$ terms in the bound of each $L_1$-norm sum together to $1$, and we are just left with the remaining $\frac{1}{2}$. The term $\epsilon_{\mc{H}}$ can be similarly estimated by performing the estimation for each class variable, and then summing and dividing by $2$.

We thus have the following theorem:
\begin{thm}\label{thm:delta_trace}
Let $\delta_{\hat{p}}^{emp}$ be the empirical estimate of $\delta_{\hat{p}}$. Then the estimator $R'R\text{proj}_{V_S}Tp(y|x)$, which requires only the training data, obtains an empirical estimate $\delta_{\hat{p}}^{emp}$ bounded via:
\begin{align}\label{eqn:delta_TED}
    &\delta_{\hat{p}}^{emp} \leq \frac{TED^{\frac{1}{2}}(\mc{S}, \mc{D}, 0)}{2N}\sqrt{\sum_i \lambda_i} \notag \\ 
    &+ \frac{1}{2}\sum_c \sqrt{\sum_i \inp{p_{y=c|x}}{\phi_i}^2(1-\lambda_i)^2 + \sum_j \inp{p_{y=c|x}}{\zeta_j}^2}
\end{align}
This further implies that any algorithm providing a better estimate of $p_{y|x}$ (in terms of conditional total variation) than $R'R\text{proj}_{V_S}Tp(y|x)$ satisfies the same bound. 
\end{thm}

\begin{proof}
    This is a direct culmination of the lemmas and discussions of this section. The bound provided is a combination of equation (\ref{eqn:all_together}), lemma \ref{lem:eh}, and the class-combining discussion of the preceding paragraph. That the estimator $R'R\text{proj}_{V_S}Tp(y|x)$ depends only on the training data points follows from lemma \ref{lemma:fin_dim_proj} and the fact that the sampled empirical estimate of $p(y=c|x)$ that we are approximating is either $0$ or $1$ for all $x \in \mc{X}, c \in \mc{Y}$.
\end{proof}

If each $p_{y=c|x} \in \mc{H}$ then we can drop the second term in the second sum to obtain:
\begin{align}
    &\delta_{\hat{p}}^{emp} \leq \frac{TED^{\frac{1}{2}}(\mc{S}, \mc{D}, 0)}{2N}\sqrt{\sum_i \lambda_i} \notag \\
    & \qquad + \frac{1}{2}\sum_c \sqrt{\sum_i \inp{p_{y=c|x}}{\phi_i}^2(1-\lambda_i)^2}
\end{align}

And we can perform the same analysis when $p_{y=c|x} \in \mc{H}$ without the decomposition step [equation (\ref{eqn:error_decomp})] to obtain the following corollary which more tightly couples the empirical $\delta_{\hat{p}}^{emp}$ to the TED objective function at the cost of a (possibly large) multiplicative term:
\begin{cor}
Take the hypotheis of of Theorem \ref{thm:delta_trace}. Assume further that each ${p_{y=c|x} \in \mc{H}}$. Then:
\begin{equation}
    \delta_{\hat{p}}^{emp} \leq \frac{TED^{\frac{1}{2}}(\mc{S}, \mc{D}, 0)}{2N}\sqrt{\sum_i \lambda_i} \sum_c \sqrt{\sum_i \frac{\inp{p_{y=c|x}}{\phi_i}^2}{\lambda_i^2}}
\end{equation}
where the final multiplicative factor can be recognized as $\sum_c \|p_{y=c|x}\|_{\mc{H}}$.
\end{cor}

\subsection{Optimization and Notes}\label{subsec:opt}
\subsubsection{A Caveat on Optimization}
A few methods of optimizing the regular TED objective function exist \cite{yu2006active, yu2008trnon}. However, in deriving a bound for $\dtv$, we've found instead a relationship to the objective function which we've denoted $TED^{\frac{1}{2}}$. Unfortunately, $\left(\text{TED}^{\frac{1}{2}}\right)^2 \neq \text{TED}$. Instead, we have:
\begin{align}
    &\left(TED^{\frac{1}{2}}\right)^2 = TED + \notag \\
    &\sum_{ij} \sqrt{(K_{x_i, x_i} - K_{x_i, S}{K_{SS}^{-1}}K_{S, x_i})(K_{x_j, x_j} - K_{x_j, S}{K_{SS}^{-1}}K_{S, x_j})}
\end{align}
and so, in optimizing the standard TED objective, we neglect these latter cross-terms. Nonetheless, optimizing the standard TED objective does still reduce the first term to a minimum, and so will still result in better quality representations. 

\subsubsection{A Caveat on Kernel Feature Dimensions}
The following property is true of TED independently of our analysis: if the feature space of the kernel, $r$, is smaller than the number of desired training data points $M$, then $TED(\mc{D}, \mc{S}, 0)$ is zero for all $\mc{S}$. To see this, note that in this case, we would be able to write $\bmat{k(x_{a_l}, x_{a_{l'}})}^l_{l'} = VV^T$ where $V\in \mathbb{R}^{N \times r}$ is a tall rank $r$ matrix. We then have the following equivalent form of the TED objective that was derived in reference \cite{yu2006active}:
\begin{equation}\label{eqn:ted2}
    \underset{B}{\min~}\underset{A}{\min} \sum_k \|V_{:,k} - ABV \|_2^2
\end{equation}
where $B$ is a diagonal $N\times N$ matrix constrained to have $M$ unit values and zeros elsewhere. But then the column space of $V$ has dimension $r$, so every column $V_{:, k}$ is a linear combination of just $r$ columns. Thus there exists a $B$ such that the objective function is zero (with $A$ corresponding to that linear combination).

This is particularly problematic for linear kernels when the original design matrix $X$ has low rank. And while most kernels have feature spaces of countably infinite dimension, it may still be problematic if we approximate our kernel through a finite kernel matrix and perform svd to directly obtain $VV^T$. It is important, then, to use at least $M$ columns of $V$ for our new feature space. 

\subsubsection{An Augmentation for Scale Invariance}
Also, we note that the training dataset returned from minimizing TED$(\cdot, \cdot, 0)$ does not depend on the scale of the kernel. To see this, we look again at (\ref{eqn:ted2}). Letting $c >0$, we then see that:
\begin{align}
    \underset{B}{\min~}\underset{A}{\min} \sum_k \|c^{\frac{1}{2}} V_{:,k} - AB c^{\frac{1}{2}}V \|_2^2 \notag \\
    = c~ \underset{B}{\min~}\underset{A}{\min} \sum_k \|X_{:,k} - ABV \|_2^2
\end{align}
Thus scaling the dataset by $c$ does not effect the chosen training samples. However, the non-greedy optimization strategy derived in reference \cite{yu2008trnon} does not share this scale invariance. But this can be fixed fairly easily by passing the (known) scaling parameter $c$ to the $A$ update of the method. This yields a new update given by:
\begin{equation}\label{eqn:aupdate}
    A \leftarrow VV^T(c \beta^{-1} + VV^T)^{-1}
\end{equation}
while the $\beta$ update remains unchanged:
\begin{equation}\label{eqn:bupdate}
    \beta_{j,j} \leftarrow \sqrt{\frac{1}{\gamma} \|A_{:,j}\|_2^2}
\end{equation}
where $\gamma$ is a regularization parameter controlling the sparsity of $\beta$. Iterating over this new $A$ update and the old $\beta$ update will give us a scale invariant ranking matrix $\beta$ in which larger values of $\beta_{j,j}$ indicate higher importance of the data-point indexed by $j$ according to the TED objective. Note that the passed scale $c$ is that of the \emph{kernel}, not of $V$.

Scale invariance allows us to tighten our bound via taking infinums over the scale parameter.
\begin{cor}
Take all assumptions and definitions of Theorem \ref{thm:delta_trace}. Let $TED^{\frac{1}{2}}(\mc{S}, \mc{D}, 0, c), c>0$ denote the value of $TED^{\frac{1}{2}}(\mc{S}, \mc{D}, 0)$ when using a scaled kernel function $ck$. Let $\epsilon_{\mc{H}}$ denote the value of $\epsilon_{\mc{H}}(c)$ when using the scaled kernel. Let $\mc{S}^*$ be the training dataset returned from a scale invariant minimization of the standard TED objective function. Then:
\begin{equation}
    \delta_{\hat{p}}^{emp} \leq \underset{c > 0}{\inf ~}\frac{TED^{\frac{1}{2}}(\mc{S}^*, \mc{D}, 0, c) { \sqrt{Tr(ck)}}}{2N}  + \frac{1}{2}\sum_c \epsilon^c_{\mc{H}}(c)
\end{equation}
where
\begin{equation}
    \epsilon_{\mc{H}}^c \triangleq \sqrt{\sum_i \inp{p_{y=c|x}}{\phi_i}^2(1-\lambda_i)^2 + \sum_j \inp{p_{y=c|x}}{\zeta_j}^2}
\end{equation}

\end{cor}
This is particularly useful because of the terms in $\epsilon_{\mc{H}}$ of the form $\inp{p_{y|x}}{\phi_i}(1-\lambda_i)$. For indices with large $\inp{p_{y|x}}{\phi_i}$,  we will likely get a much better bound if our kernel has $\lambda_{i} \approx 1$. With scale invariance, this can be made true without explicitly scaling our kernel, which is good because we typically won't know the coefficients $\inp{p_{y|x}}{\phi_i}$ (at least not without a small amount of initial labels).  

\section{Experiments}

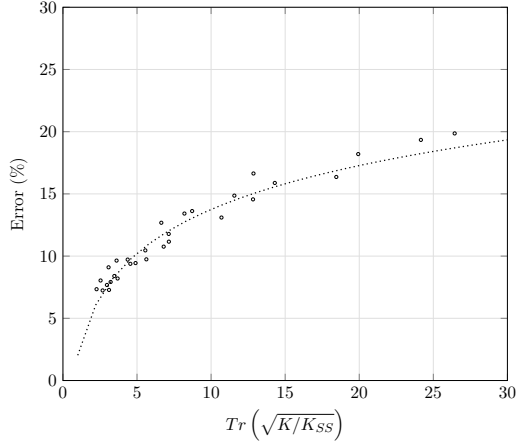
\begin{figure}
    \centering
    \begin{subfloat}{}
 \begin{tikzpicture}[scale=0.8]
    \begin{axis}[xlabel=$Tr\left(\sqrt{K/K_{SS}}\right)$,
	xmin=0,
    xmax=30,
    ymin=0,
    ymax=30,
	ylabel= Error (\%),
	grid=both,
	minor grid style={gray!25},
	major grid style={gray!25},
	width=0.75\linewidth,
    scatter/classes={a={mark=o,draw=black, mark size=1pt}}]]
    
\addplot[scatter,only marks,%
    scatter src=explicit symbolic]%
    table[meta=label] {
    x y label
    26.44	19.86 a
    24.16	19.34 a
    19.93	18.2 a
    18.45	16.36 a
    14.3	15.88 a
    12.86	16.64 a 
    11.57	14.86 a
    8.72	13.62 a
    12.83	14.56 a
    8.2	    13.42 a
    10.7	13.1 a
    6.64	12.68 a
    7.14	11.78 a
    7.15	11.16 a
    5.56	10.46 a
    6.8	    10.76 a
    5.63	9.74 a
    4.55	9.38 a
    4.9	    9.44 a
    3.62	9.64 a
    4.37	9.72 a
    3.08	9.1 a
    3.69	8.2 a
    3.47	8.4 a
    2.54	8.04 a
    2.97	7.68 a
    2.27	7.34 a
    3.22	7.92 a
    2.68	7.24 a
    3.11	7.28 a
    };
\addplot+[name path=B,black, no marks, domain=1:30, dotted, thick] {5.09*ln(x) + 2.03};
    
    \end{axis}
    \end{tikzpicture} 
\end{subfloat}
    \caption{$TED^{\frac{1}{2}}$ objective value vs classification accuracy on MNIST with the cosine kernel.}
    \label{fig:correspondence}
\end{figure}

\begin{figure}[t]
    \centering
    \includegraphics[width=0.49\textwidth, height=8cm]{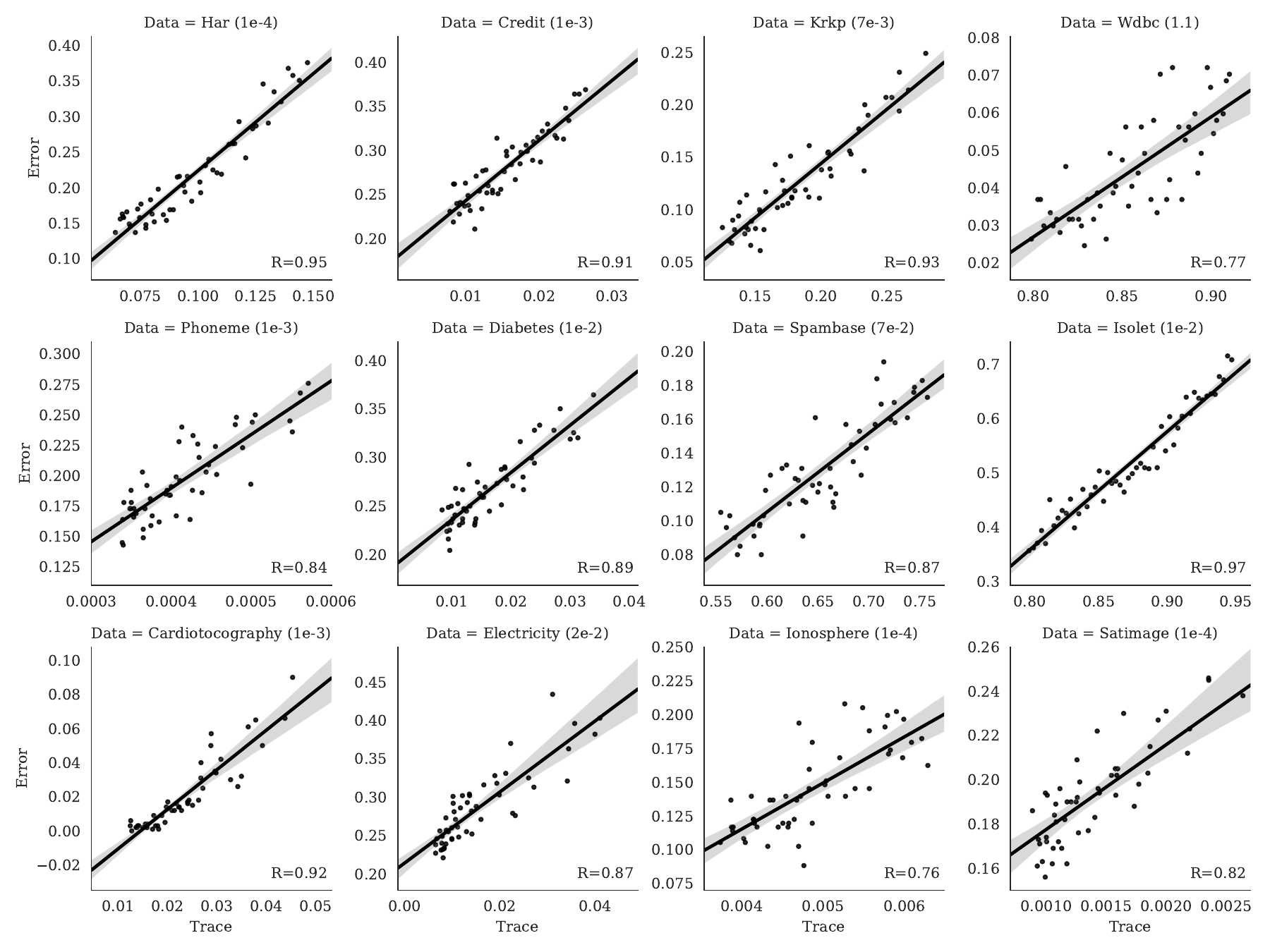}
    \caption{Classification errors against the data quality measure with a fixed training data size for varying datasets. Dataset and corresponding rbf $\gamma$ value are indicated at the top of each plot.}
    \label{fig:scatter_plots}
\end{figure}
\begin{figure}[t]
\centering
    \includegraphics[width = 0.49\textwidth, height=4cm]{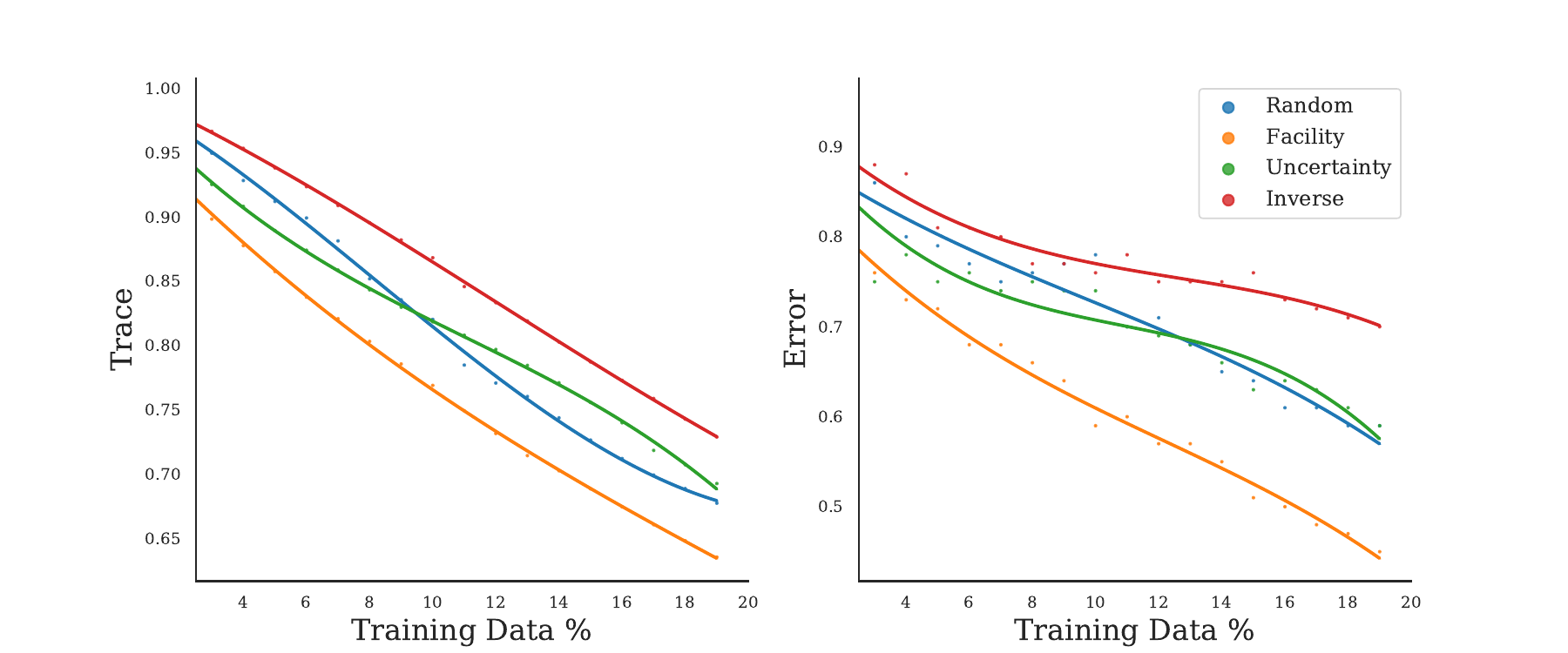}
    \caption{MNIST data quality measure and classification error against training data size for several methods of training data selection.}
    \label{fig:tr_sweep_cmpr}
\end{figure}

We first tested if our bound had any generalizable meaning to classification accuracies. To do so, we first took the MNIST with the cosine kernel and swept over training data sizes from $5\%$ to $90\%$ and plotted the trace term against the classification accuracy of a fully connected feed forward neural network with $1000$ hidden units in Figure \ref{fig:correspondence}. We see a strong correspondence between these two variables. However, to ensure that this correspondence is meaningful, we also need to ensure that it exists when the training data size is controlled - since both terms, in isolation, decrease as $M$ grows. 

To do this, we took several samples of training data over a variety of datasets provided by OpenML \cite{OpenML2013} and trained them on a fully connected feed-forward neural network with $1000$ hidden units - plotting the trace term against the classification error in each case. In each case, we took $20\%$ of the full dataset as training data. Similar correspondences occur at any percentage of training data. We have provided the resulting scatter plots in Figure \ref{fig:scatter_plots} which show that, in this controlled scenario, the two variables are still correlated. Each point in these plots corresponds to a different ratio of selected points to random points in that fixed-size training dataset, where selection is done with an `inverse' heuristic method which just picks data corresponding to small diagonal elements in $K^{-1}$.

A second way of ensuring that the relationship found in Figure \ref{fig:correspondence} is meaningful is to observe the behavior of different data selection methods over a training data sweep. We have done this in Figure \ref{fig:tr_sweep_cmpr}. On the left hand side of Figure \ref{fig:tr_sweep_cmpr}, we have plotted training data size against the trace term (for MNIST under the cosine kernel) for five such methods of training data selection: random selection, facility location, uncertainty sampling, and the inverse heuristic from the last paragraph. On the right hand side of Figure \ref{fig:tr_sweep_cmpr}, we have a similar plot where the trace term is replaced by classification error of a fully connected feed forward neural network with $1000$ hidden units. We see that the behavior of the trace term plots are carried over to the classification error plots. Some small scale information is lost, mostly due to the fact that the error plot is more noisy, but the main global properties are intact. This correspondence of behavior further shows that there is a link, independent of training data size, between our bound and classification error.

\section{Conclusion}
This paper has provided a novel information theoretic perspective on active learning methods. It has provided an information theoretic proof of the viability of the facility location function data selection method, and derived a new information theoretic bound, written in terms of the objective function of Transductive Experimental Design, which is highly applicable to evaluating and analyzing other active learning strategies. Experiments show that this bound is indicative of dataset quality in terms of classification accuracies.

\bibliographystyle{IEEEtran}
\bibliography{references}
\end{document}